\documentclass{article}

\PassOptionsToPackage{numbers,sort}{natbib}

\usepackage[utf8]{inputenc}
\usepackage[final]{neurips_2023}

\usepackage{amsmath}
\usepackage{amsthm}
\usepackage{amssymb}
\usepackage{booktabs}
\usepackage{dsfont}
\usepackage{graphicx}
\usepackage{microtype}
\usepackage{mleftright}
\usepackage{multicol}
\usepackage{wrapfig}
\usepackage{algorithm}
\usepackage{algpseudocode}

\usepackage{paralist}
\setdefaultleftmargin{1em}{1em}{1em}{1em}{1em}{1em}

\usepackage{siunitx}
\usepackage{subcaption}
\usepackage{thm-restate}
\usepackage{xspace}
\usepackage{multirow}
\usepackage{pdfpages}
\usepackage{subfiles}

\usepackage{tikz}
\usepackage{tikz-network}
\usetikzlibrary{calc}
\usetikzlibrary{arrows}
\usetikzlibrary{calc}
\usetikzlibrary{chains}
\usetikzlibrary{fit}
\usetikzlibrary{matrix}
\usetikzlibrary{positioning}
\usetikzlibrary{scopes}
\usetikzlibrary{external}

\usepackage{pgfplots}
\usepackage{pgfplotstable}
\usepackage{subcaption}
\usepgfplotslibrary{groupplots}
\usepackage{siunitx}
\usepgfplotslibrary{colorbrewer}
\pgfplotsset{cycle list/Dark2-8}
\usetikzlibrary{patterns}

\pgfplotsset{
  compat = 1.18,
}

\definecolor{bleu}     {RGB}{ 49,140,231}
\definecolor{cardinal} {RGB}{196, 30, 58}
\definecolor{emerald}  {RGB}{ 80,200,120}
\definecolor{lightgrey}{RGB}{230,230,230}
\definecolor{edge1}{rgb}{1.         0.85359477 0.06797386}
\definecolor{edge2}{rgb}{1.         0.85359477 0.06797386}
\definecolor{edge3}{rgb}{0.75105883 0.84 0.40188235}
\definecolor{edge4}{rgb}{1.         0.74028451 0.        }
\definecolor{edge5}{rgb}{1.         0.66535948 0.        }
\definecolor{edge6}{rgb}{0.99598616 0.58329873 0.        }
\definecolor{edge7}{rgb}{0.88199369 0.85918714 0.13301246}
\definecolor{edge8}{rgb}{0.98823529 0.49803922 0.        }
\definecolor{edge9}{rgb}{1.         0.74028451 0.        }
\definecolor{edge10}{rgb}{1.         0.85359477 0.06797386}

\DeclareRobustCommand\legendbox[1]{\textcolor{#1}{}\begin{tikzpicture}[x=0.2cm, y=0.2cm] \draw [color=black, fill=#1!20,postaction={pattern=crosshatch}] (0,0) -- (0,1) -- (0.6,1) -- (0.6,0) -- (0, 0); \end{tikzpicture}}

\newcommand{\reals}{\mathds{R}\xspace}
\newcommand{\FRC}  {\ensuremath{\kappa_{\text{FR}}\xspace}}
\newcommand{\ORC}  {\ensuremath{\kappa_{\text{OR}}\xspace}}
\newcommand{\REC}  {\ensuremath{\kappa_{\text{R}}\xspace}}

\DeclareMathOperator{\degree}{deg}

\newtheorem{lemma}      {Lemma}
\newtheorem{proposition}{Proposition}

\newcommand{\highlight}[1]{%
  \fcolorbox{black}{bleu!25}{%
    \parbox{\dimexpr\linewidth-2\fboxsep-2\fboxrule}{%
    #1%
    }%
  }%
}

\usepackage{hyperref}

\hypersetup{
  colorlinks = true,
  urlcolor   = bleu,
  linkcolor  = bleu,
  citecolor  = bleu,
}

\usepackage[capitalize, nameinlink]{cleveref}

\begin{document}

\title{%
  Curvature Filtrations for Graph Generative Model Evaluation
}

\author{%
  Joshua Southern\footnotemark[1]\\
  Imperial College London \\
  \texttt{jks17@ic.ac.uk} \\
  \And
  Jeremy Wayland\thanks{These authors contributed equally.} \\
  Helmholtz Munich \& Technical University of Munich\\
  \texttt{jeremy.wayland@tum.de} \\
  \AND
  Michael Bronstein\footnotemark[2]\\
  University of Oxford\\
  \texttt{michael.bronstein@cs.ox.ac.uk} \\
  \And
  Bastian Rieck\thanks{These authors jointly directed the work.}\\
  Helmholtz Munich \& Technical University of Munich\\
  \texttt{bastian.rieck@tum.de} \\
}

\maketitle

\begin{abstract}
  Graph generative model evaluation necessitates understanding
  differences between graphs on the distributional level.
  This entails being able to harness salient attributes of graphs in an efficient
  manner.
  Curvature constitutes one such property that has recently proved its utility in characterising graphs.
  Its expressive properties, stability, and practical utility in model
  evaluation remain largely unexplored, however.
  We combine graph curvature descriptors with emerging methods from
  topological data analysis to obtain robust, expressive descriptors for
  evaluating graph generative models.
\end{abstract}


\section{Introduction}

Graph-structured data are prevalent in a number of different domains,
including social networks~\citep{fake_news}, bioinformatics
\citep{Jumper2021, Ruiz2021-ac}, and transportation
\citep{Jiang_2022}.
The ability to generate new graphs from
a distribution is a burgeoning technology with promising applications in
molecule or protein design~\cite{junction-tree, antibody-design}, as
well as circuit design~\citep{circuit-design}.
To compare graph generative models and advance research in this
area, it is essential to have a metric that can measure the distance
between a set of generated and reference graphs, with enough \emph{expressivity}
to critically evaluate model performance.
Typically, this is done by using a set of descriptor functions, which
map a graph to a high-dimensional representation in~$\reals^d$. An
evaluator function, such as the \emph{maximum mean
discrepancy}~\citep[MMD]{mmd}, may then be used to get a distance between
two distributions of graphs by comparing their vectorial
representations~\citep{Liao, Niu}.
This state of the art was recently critiqued by \citet{OBray22a} since
it
\begin{inparaenum}[(i)]
  \item requires numerous parameter and function choices, 
  \item is limited by the \emph{expressivity} of the descriptor function, and
  \item does not come equipped with \emph{stability guarantees}. 
\end{inparaenum}

We propose to overcome these issues through topological data
analysis~(TDA), which is capable of capturing multi-scale features of
graphs while being more expressive than simple descriptor functions.
TDA is built on the existence of a function of the form~$f\colon V \to \reals$, or~$f\colon E \to \reals$ 
on a graph~$G = (V, E)$. This is used to obtain a \emph{filtration},
i.e.\ an ordering of subgraphs, resulting in a set of
topological descriptors, the \emph{persistence diagrams}. Motivated by
their expressive power and computational efficiency, we use recent
notions of discrete curvature~\citep{Devriendt22, Forman2003, OLLIVIER2007643} 
to define filtrations and calculate \emph{persistence
landscapes}~\citep{Bubenik15} from the persistence diagrams, thus
obtaining a descriptor whose Banach space formulation permits
statistical calculations.
Our proposed method comes equipped with stability guarantees, can count
certain substructures and measure important graph characteristics, and
can be used to evaluate a variety of statistical tests since it permits
computing distances between graph distributions.

Our \textbf{contributions} are as follows:
\begin{compactitem}
  \item We provide a thorough theoretical analysis of the stability and
    expressivity of recent notions of discrete curvature, showing their fundamental utility for graph learning applications.
  \item Using discrete curvature and TDA, we develop a new metric for
    graph generative model evaluation.
  \item Our experiments reveal our metric is robust and expressive, thus
    improving upon current approaches that use simple graph descriptor
    and evaluator functions. 
\end{compactitem}

\section{Background}

The topological descriptors used in this paper are based on
\emph{computational topology} and \emph{discrete curvature}.
We give an overview of these areas and briefly comment on previous work
that makes use of \emph{graph statistics} in combination with
MMD. In the following, we consider undirected graphs, denoted by $G = (V,
E)$, with a set of vertices~$V$ and a set of edges~$E \subseteq V \times
V$.

\subsection{MMD and Metrics Based On Graph Statistics}

MMD is a powerful method for comparing different distributions of
structured objects. It employs \emph{kernel functions}, i.e.\ positive
definite similarity functions, and can thus be directly combined with
standard graph kernels~\citep{Borgwardt20} for graph distribution
comparison.
%
However, there are subtle issues when calculating kernels in
$\reals^d$: Gaussian kernels on geodesic metric spaces, for instance, have
limited applicability when spaces have non-zero
curvature~\citep{Feragen15}, and certain kernels in the literature are
indefinite, thus violating one of the tenets of MMD~\citep{OBray22a}.
Despite these shortcomings, MMD is commonly used to evaluate graph
generative models~\citep{graph-rnn, sparse-graph-gen}. This is
accomplished by extracting a feature vector from each graph, such as the
clustering coefficient or node degree, and subsequently calculating
empirical MMD values between generated and reference samples.
Some works~\citep{moreno} combine multiple structural properties
of graphs into a single metric through the Kolmorogov--Smirnov~(KS)
multidimensional distance~\citep{JUSTEL1997251}. Combining graph
statistics into a single measure has also led to metrics between
molecular graphs, such as the quantitative estimate of
drug-likeness~(QED), which is a common measure in drug
discovery~\citep{qed}. However, these simple statistics, even when
considered jointly, often lack expressivity, have no stability
guarantees, and their use with MMD requires numerous parameter and
function choices~\citep{OBray22a}.

\subsection{Computational Topology}


Computational topology assigns \emph{invariants}---characteristic
properties that remain unchanged under certain transformations---to
topological spaces. For graphs, the simplest invariants are given by the 
\mbox{$0$-dimensional}~($\beta_0$) and \mbox{$1$-dimensional}~($\beta_1$) Betti
numbers. These correspond to the number of connected components and number of
cycles, respectively, and can be computed efficiently. Their limited
expressivity can be substantially increased when paired with a scalar-valued \emph{filtration function} $f\colon E \to \reals$.\footnote{One can also consider functions over vertices, $f\colon V \to \reals$, when building filtrations. See \cref{app:sec:edge_v_vertex_filtrations} for a remark on the equivalence of these viewpoints.}
Since~$f$ can only attain a finite
number of values~$a_0,a_1, a_2, \dots$ on the graph, this permits
calculating a \emph{graph filtration}
$
\emptyset \subseteq G_0 \subseteq G_1 \ldots \subseteq G_{k-1} \subseteq G_k = G,
$
where each $G_i := (V_i, E_i)$, with $E_i := \{ e \in E \mid f(e) \leq a_i\}$
and $V_i := \{ v \in V \mid \exists e\in E_i \text{ s.t. } v\in e\}$.
This \emph{sublevel set filtration}\footnote{%
  Swapping `$\leq$' for `$\geq$' and $\max$ for $\min$ results in the
  \emph{superlevel set filtration} of equal expressivity.
}
permits tracking topological features, such as cycles, via \emph{persistent
homology}~\citep{Edelsbrunner10}. If a topological feature appears for the first time in $G_i$ and
disappears in $G_j$, we represent the feature as a tuple $(a_i, a_j)$,
which we can collect in a \emph{persistence diagram}~$D$.
Persistent homology thus tracks changes in connected components and
cycles over the complete filtration, measured using a filtration
function~$f$.
Persistence diagrams form a metric space, with the distance between them
given by the \emph{bottleneck distance}, defined as 
$d_B\mleft(D, D'\mright) := \mleft( \inf_{\eta\colon D \to D'}\sup_{x\in
D}\|x-\eta(x)\|_\infty \mright)$,
where $\eta$ ranges over all bijections between the two diagrams.
A seminal \emph{stability theorem}~\citep{Cohen-Steiner2007} states
that the bottleneck distance between persistence diagrams~$D_f,
D_g$, generated from two functions~$f$ and~$g$ on the same graph, is
upper-bounded by $d_B(D_f,D_g) \leq ||f-g||_\infty$.
The infinity norm of the functions, a geometrical quantity, hence
limits the topological variation.
In practice, we convert persistence diagrams to an equivalent representation, the \emph{persistence landscape}~\citep{Bubenik15}, which is more amenable to statistical analyses and the integration into machine learning pipelines.

\paragraph{Advantageous Properties.}
%
Persistent homology satisfies expressivity and stability properties. The
choice of filtration~$f$ affects expressivity: with the right
filtration, persistent homology can be \emph{more} expressive than the
$1$-WL test~\citep{Horn22a, rieck2023expressivity}, which is commonly used to bound the
expressivity of graph neural networks~\citep{Morris19a, Morris21a, Xu19a}.
Thus, given a suitable filtration, we can create a robust and expressive
metric for comparing graphs. Moreover, as we will later see, TDA
improves the expressivity of \emph{any} $f$, meaning
that even if the function on its own is not capable of distinguishing
between different graphs, using it in a filtration context can overcome
these deficiencies.

\subsection{Discrete Curvature}

While filtrations can be learnt~\cite{Hofer20, Horn22a}, in the absence of
a well-defined learning task for graph generative model evaluation, we
opt instead to employ existing functions that exhibit suitable
expressivity properties.
Of particular interest are functions based on \emph{discrete curvature},
which was shown to be an expressive feature for graph and molecular learning
tasks~\citep{affinity-gnn, ollivier_binding_affinity, forman_binding_affinity}. 
Curvature is a fundamental concept in differential geometry and
topology, making it possible to distinguish between different types of
manifolds. There are a variety of different curvature formulations with varying
properties, with \emph{Ricci curvature} being one of the most
prominent.
Roughly speaking, Ricci curvature is based on measuring the differences in
the growth of volumes in a space as compared to a model Euclidean
space. While originally requiring a smooth setting, recent work  started
exploring how to formulate a theory of Ricci curvature in
the discrete setting~\citep{Coupette23a, Haantjes, OLLIVIER2007643, Forman2003,
bakry-emery, Devriendt22}.
Intuitively, discrete curvature measures quantify a notion of
similarity between node neighbourhoods, the discrete concept
corresponding to `volume.' They tend to be \emph{larger} for
structures where there are overlapping
neighbourhoods such as cliques, \emph{smaller}~(or zero) for grids and
\emph{lowest}~(or negative) for tree-like structures.
Ricci curvature for graphs provides us with sophisticated tools to
analyse the neighbourhood of an edge and recent works have shown the
benefits of using some of these curvature-based methods in combination
with Graph Neural Networks~(GNNs) to boost
performance~\citep{affinity-gnn}, assess differences between real-world
networks~\citep{Samal_2018}, or enable \emph{graph rewiring} to reduce
over-squashing in GNNs~\citep{over-squashing-curvature}. However, the
representational power and stability properties of these measures remain
largely unexplored.
Subsequently, we will focus on three different types of curvature, 
\begin{inparaenum}[(i)]
  \item Forman--Ricci curvature~\citep{Forman2003},
  \item Ollivier--Ricci curvature~\citep{OLLIVIER2007643}, and
  \item Resistance curvature~\cite{Devriendt22}.
\end{inparaenum}
We find these three notions to be prototypical of discrete curvature
measures, increasing in complexity and in their ability to capture
\emph{global} properties of a graph.

\paragraph{Forman--Ricci Curvature.}

The \emph{Forman(--Ricci) curvature} for an edge $(i, j)\in E$ is defined as 
\begin{equation}
  \FRC(i,j) := 4 - d_i - d_j + 3|\#_{\Delta}|,
  \label{eq:Forman}
\end{equation}
where $d_i$ is the degree of node $i$ and $|\#_{\Delta}|$ is the number
of $3$-cycles~(i.e.\ triangles) containing the adjacent nodes.

\paragraph{Ollivier--Ricci Curvature.}
%
Ollivier introduced a notion of curvature for metric spaces that
measures the Wasserstein distance between Markov chains, i.e.\ random
walks, defined on two nodes~\citep{OLLIVIER2007643}.
Let $G$ be a graph with some metric $d_{G}$, and $\mu_v$ be
a probability measure on G for node $v \in V$.
The \emph{Ollivier--Ricci
curvature} of \emph{any}\footnote{%
  In contrast to other notions of curvature, Ollivier--Ricci curvature
  is defined for \emph{both} edges and non-edges.
}
pair ${i, j}\in V \times V$ with $i \neq j$ is then defined as
\begin{equation}
  \ORC(i,j) := 1 - \frac{1}{d_G(i, j)}W_{1}(\mu_{i}, \mu_{j}),
  \label{eq:Ollivier--Ricci}
\end{equation}
where $W_1$ refers to the first \emph{Wasserstein distance} between
$\mu_i$, $\mu_j$.
\cref{eq:Ollivier--Ricci} defines the Ollivier--Ricci~(OR) curvature
in a general setting outlined by
\citet{or_curvature_smooth}; this is in contrast to the majority of
previous works in the graph setting which specify $d_{G}$ to be
the shortest-path distance and $\mu_i, \mu_j$ to be uniform
probability measures in the $1$-hop neighbourhood of the node.
Extending the probability measures so that they act on larger locality
scales is known to be beneficial for characterising graphs~\citep{Gosztolai2021,
k-step-measure, infinite-reg-trees}.
We assume this general setting to define different notions of OR
curvature on the graph, permitting us the flexibility of altering the
probability measure and the metric. 

\paragraph{Resistance Curvature.}
%
The resistance curvature for edges of a graph, as established in
\citet{Devriendt22}, is inspired by Ohm's Law and the concept of
effective resistance, a well-studied, global metric between nodes in a weighted
graph that quantifies the resistance exerted by the network when current
flows between nodes.
For a graph $G=(V,E)$, let $R_{ij}$ be the resistance distance between
nodes $i,j\in V$, defined in \cref{eq:resistance_distance}.
The \emph{node resistance curvature} of a node $i \in V$ is then
defined as
$p_i := 1-\frac{1}{2} \sum_{j\sim i}R_{ji}$.
The curvature of an edge $(i,j)\in E$, what we refer to as
\emph{resistance curvature}, is then defined as
\begin{equation}\label{eq:Resistance}
  \REC(i,j) := \frac{2(p_i+p_j)}{R_{ij}}
\end{equation}
The resistance curvature of an edge is related to the average
distance between the neighbourhoods of the nodes connected by the edge. 

\section{Our Method}
%
Each notion of discrete curvature yields a scalar-valued function on the edges of a graph.
We use these functions to define a family of \emph{curvature
filtrations} based on sublevel sets.
In combination with persistent homology, this enables us to assess the
structural properties of a given graph at multiple scales, measured via curvature. 
Using
metrics on aggregated topological signatures---here, in the form of
\emph{persistence landscapes}---we may then compare two distributions
of graphs.
Specifically, we propose the following scheme for graph generative model
evaluation:
\begin{compactenum}
  \item Given a specific curvature filtration, we generate a set of
    \emph{persistence diagrams} that encode the persistent
    homology in dimensions~$0$ and~$1$ for each graph in the distribution. Each diagram tracks the lifespan of connected components and
     cycles as they appear in the filtration of a given graph, resulting in a multi-scale
     summary of the graph's structure. 
  \item To permit an analysis on the distributional level, we convert
    each diagram into a more suitable representation, namely
    a \emph{persistence landscape}. As functional summaries of topological information, persistence
    landscapes allow for easy calculation of averages and
    distances~\citep{Bubenik15}.
  \item Finally, we conduct statistical comparisons
    between graph distributions, e.g.\ permutation tests using $p$-norms, in this latent space, providing an \emph{expressive} metric for evaluating generative models.
\end{compactenum}
\Cref{fig:pipeline} illustrates our proposed pipeline using Forman--Ricci curvature. For choosing a notion of curvature in practice, implementation details, and computational performance we refer the reader to \cref{sec:ChoosingCurvature}, \cref{app:pseudocode}, and \cref{app:sec:Computational Complexity} respectively. We also make our framework publicly available.\footnote{Source code is available at \url{https://github.com/aidos-lab/CFGGME}.}

\begin{figure*}[tbp]
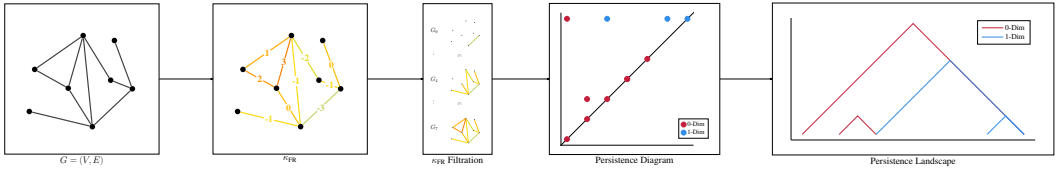

  \centering
  \subfile{figures/main_figure}
  \caption{
    An overview of our pipeline for evaluating graph generative models
    using discrete curvature. We depict a graph's edges being coloured by
    $\FRC$, as described in \cref{eq:Forman}.
    The ordering on edges gives rise to a \emph{curvature
    filtration}, followed by a corresponding persistence diagram and
    landscape. For graph generative models, we select a curvature, apply
    this framework element-wise, and evaluate the similarity of the
    generated and reference distributions by comparing their average
    landscapes.
  }
  \label{fig:pipeline}
\end{figure*}

\subsection{Stability}\label{sec:Stability}

We first discuss the general stability of topological calculations,
proving that changes in topological dissimilarity are bounded by changes
in the filtration functions. Moreover, we show that
filtrations based on discrete curvature satisfy stability
properties if the underlying graph is modified.

\paragraph{Topological Features.}
%
Using the persistent homology stability
theorem~\citep{Cohen-Steiner2007}, we know that if two curvature
filtrations are similar on a graph, their persistence
diagrams will also be similar.
However, the theorem only holds for two different functions on
the \emph{same} graph, whereas the distributional case has not yet been
addressed by the literature. Our theorem provides a~(coarse) upper bound.
\begin{restatable}{theorem}{thmstabilityupperbound}
  Given graphs $F = (V_F, E_F)$ and $G = (V_G, E_G)$ with
  filtration functions~$f, g$, and corresponding persistence
  diagrams $D_f, D_g$, we have
  $d_B(D_f, D_g) \leq \max\mleft\{ \mathrm{dis}(f, g), \mathrm{dis}(g, f) \mright\}$,
  where $\mathrm{dis}(f, g) := \mleft|\max_{x \in E_F} f(x) - \min_{y
  \in E_G} g(y) \mright|$ and vice versa for $\mathrm{dis}(g, f)$.
  \label{thm:Stability}
\end{restatable}
This upper bound implies that changes on the level of persistence
diagrams are bounded by changes between the input functions. The
stability of our method thus hinges on the stability of the filtration
functions, so we need to understand the behaviour of curvature under
perturbations.
Following previous work~\citep{OBray22a}, we aim to understand and
quantify the stability of our method in response to adding and deleting
edges in the graph.
Our stability theorems establish bounds on various discrete curvature
measures for \emph{finite}, \emph{unweighted}, \emph{connected} graphs
in response to these perturbations.
We restrict the outcome of a perturbation to graphs of the
form $G'=(V,E')$ that satisfy $|E| \not= |E'|$ and do not change the
number of connected components of a graph.
Our theoretical results bound the new curvature $\kappa'$ according to
the structural properties of $G$. For an exhaustive list of theorems and
proofs, as well as experiments analysing the change in curvature for
perturbed Erdős–Rényi~(ER) graphs, see
\cref{app:Proofs} and \cref{app:Stability Analysis}.

\paragraph{Forman--Ricci Curvature.}
%
We first analyse Forman--Ricci curvature ~$\FRC$, and prove that it is stable with respect to adding and deleting
edges.
\begin{restatable}{theorem}{thmFormanAdditionAndDeletion}
    If $G'$ is the graph generated by \textbf{edge addition}, then the
    updated Forman curvature $\FRC'$ for pre-existing edges $(i,j)\in E$
    can be bounded by $\FRC(i,j) -1 \leq \FRC'(i,j) \leq \FRC(i,j) + 2$.
    If $G'$ is the graph generated by \textbf{edge deletion}, then the
    updated Forman curvature $\FRC'$ for pre-existing edges $(i,j)\in E$
    can be bounded by $\FRC(i,j) -2 \leq \FRC'(i,j) \leq \FRC(i,j) + 1$.
\end{restatable}

\paragraph{Ollivier--Ricci Curvature.}
%
Let $\mathcal{G}= (G, d_G,\mu)$ be a triple for specifying
Ollivier--Ricci curvature calculations, with $G$ denoting an
unweighted, connected graph, $d_G$ its associated graph metric, and
$\mu :=\{\mu_v \mid v\in V\}$ a collection of probability measures at
each node.
Furthermore, let $\delta_i$ denote the Dirac measure at node
$i$ and $J(i)$ := $W_1(\delta_i, \mu_i)$ the corresponding jump
probability in the graph $G$ as defined by \citet{OLLIVIER2007643}.
Following an edge addition or deletion, we consider an updated triple
$\mathcal{G}' = (G', d_{G'}, \mu')$, and remark that this yields an
updated Wasserstein distance $W_1'$, calculated in terms of the new
graph metric $d_{G'}$.

\begin{restatable}{theorem}{thmORBonnetMyers}
  Given a perturbation~(either \textbf{edge addition} or \textbf{edge
  deletion}) producing $\mathcal{G}'$, the Ollivier--Ricci curvature $\ORC'(i, j)$
  of a pair~$(i, j)$ can be bounded via
  \begin{equation}
    1 - \frac{1}{d_{G'}(i,j)}\big[2W'_{\max}+W_1'(\mu_i,\mu_j)\big] \leq
    \ORC'(i, j)
    \leq \frac{J'(i)+J'(j)}{d_{G'}(i,j)},
  \end{equation}
  where $J'(v) := W_1'(\delta_v,\mu'_v)$ refers to the new jump
  probabilities and $W'_{\max} := \max_{x\in V} W_1'(\mu_x,\mu_x')$
  denotes the maximal reaction to the perturbation~(measured using the
  updated Wasserstein distance).
\end{restatable}

\paragraph{Resistance Curvature.}
%
Let $G$ be an unweighted, connected graph with a resistance distance
$R_{ij}$ and resistance curvature
$\REC(i,j)$ for each $(i,j)\in E$. Furthermore, let $d_x$ denote the
degree for node $x\in V$. 
We find that $\REC$ is well-behaved
under these perturbations, in the sense that edge additions can only
\emph{increase} the curvature, and edge deletions can only
\emph{decrease} it. For edge additions, we obtain the following
bound~(see \cref{app:Resistance} for the corresponding bound for edge
deletions).
\begin{restatable}{theorem}{thmResistanceAddition}\label{thm:ResistanceAddition}
If $G'$ is
the graph generated by \textbf{edge addition}, then $\REC'\geq
\REC$,with the following bound:
\begin{equation}
      |\REC'(i,j) - \REC(i,j)| \leq \frac{\Delta_{\mathrm{add}}(d_i+d_j)}{R_{ij}-\Delta_{\mathrm{add}}},
\end{equation} 
where $\Delta_{\mathrm{add}} := \max_{i,j\in V}\big(R_{ij}- \frac{1}{2}\big(\frac{1}{d_i+1}+\frac{1}{d_j+1}\big)\big)$.
\end{restatable}%
\noindent%
\highlight{%
The implications of this section are that
\begin{inparaenum}[(i)]
  \item the stability of our topological calculations largely hinges on
    the stability of the functions being used to define said
    filtrations, and 
  \item \emph{all} discrete curvature measures satisfy stability
    properties with respect to changes in graph connectivity,
    making curvature-based filtrations highly robust.
\end{inparaenum}
}

\subsection{Expressivity}\label{sec:Expressivity}

A metric between distributions should be non-zero when the distributions
differ. For this to occur, our metric needs to be able to distinguish
non-isomorphic graphs and be sufficiently expressive. \citet{Horn22a}
showed that persistent homology with an appropriate choice of filtration
is strictly more expressive than $1$-WL, the $1$-dimensional
Weisfeiler--Le(h)man test for graph isomorphism.
A similar expressivity result can be obtained for using resistance
curvature as a node feature~\citep{affinity-gnn}, underlining the
general utility of curvature.
We have the following general results concerning the expressivity or
discriminative power of our topological representations.
\begin{restatable}{theorem}{thmstabilitylowerbound}
  Given two graphs~$F = (V_F, E_F)$ and $G = (V_G, E_G)$ with scalar-valued
  filtration functions~$f, g$, and their respective persistence
  diagrams $D_f, D_g$, we have
  $
    d_B(D_f, D_g) \geq \inf_{\eta\colon E_F \to E_G} \sup_{x \in E_F} |f(x) - g(\eta(x))|,
  $
  where $\eta$ ranges over all maps from~$E_F$ to~$E_G$.
  \label{thm:Expressivity}
\end{restatable}
\cref{thm:Expressivity} implies that topological distances are generally
more discriminative than the distances between the filtration functions.
Thus, calculating topological representations of graphs based on a class
of functions improves discriminative power.
To further understand the expressive power of curvature filtrations, we
analyse strongly-regular graphs, which are often used for studying GNN
expressivity as they cannot be distinguished by \mbox{$k$-WL}, the
$k$-dimensional Weisfeiler--Le(h)man test, if $k \leq 3$~\citep{arvind,
bodnar, Morris21a}. Additionally, we explore how curvature filtrations
can count substructures, an important tool for evaluating and comparing
expressivity~\citep{papp}.
To the best of our knowledge, ours is the first work to explore discrete
curvature and curvature-based filtrations in this context. 

\paragraph{Distinguishing Strongly-Regular Graphs.}
%
Strongly-regular graphs are often used to assess the expressive power of
graph learning algorithms, constituting a class of graphs that are
particularly hard to distinguish.
We briefly recall some definitions.
A connected graph $G$ with diameter $D$ is called
\emph{distance regular} if
there are integers $b_i$, $c_i$, ($0 \le i \le D$) such that for any two
vertices $x, y \in V$ with $d(x, y) = i$, there are $c_i$ neighbours
of $y$ in $k_{i-1}(x)$ and $b_i$ neighbours of $y$ in $k_{i+1}(x)$. For
a distance-regular graph, the intersection array is given by $\{ b_0,
b_1, \dots , b_{D-1}; c_1, c_2, \dots , c_D\}$.
A graph is called \emph{strongly regular} if it is distance regular and
has a diameter of 2~\citep{Van_Dam_2016}. 
We first state two theoretical results about curvature.

\begin{restatable}[Expressivity of curvature notions]{theorem}{thmExpressivityCurvature}
  Both Forman--Ricci curvature and Resistance curvature \emph{cannot}
  distinguish distance-regular graphs with the same intersection array,
  whereas Ollivier--Ricci curvature \emph{can} distinguish the Rook and
  Shrikhande graphs, which are strongly-regular graphs with the same
  intersection array.
\end{restatable}
The Rook and Shrikhande graph \emph{cannot} be distinguished by
$2$-WL~\citep{gsn, Morris21a}.
However, OR curvature is sensitive to differences in their
first-hop peripheral subgraphs~\citep{khop-gnn}, thus
distinguishing them. This result shows the limitations of Forman--Ricci
and Resistance curvature, as well as the benefits of using
Ollivier--Ricci curvature.
We show further experiments with curvature notions on strongly-regular
graphs in the experimental section, observing improvements whenever
we use them as filtrations.

\paragraph{Counting Substructures.}
%
Evaluating the ability of curvature to encode structural information is
a crucial aspect for understanding its expressivity and validating its
overall utility in graph learning. It has previously been shown that
incorporating structural information of graphs can extend the
expressive power of message-passing neural networks~\citep{powerful_gnn,
geerts, local_heirarchy}. Additionally, \citet{gsn} show how GNNs can be
\begin{inparaenum}[(i)]
  \item strictly more expressive than $1$-WL by counting local
    isomorphic substructures~(e.g.\ cliques), and
  \item exhibit predictive performance improvements when adding such
    substructure counts to the node features.
\end{inparaenum}
For instance, some strongly-regular graphs can be distinguished by
counting $4$-cliques. Discrete curvature measures are informed by these
local substructures and have been shown to improve expressivity beyond
1-WL~\citep{affinity-gnn} when included as a node feature. 
Nevertheless, there is limited work exploring what substructure information curvature
carries, which would allow us to describe the expressivity of the
measure.
Moreover, persistent homology can track the number of
cycles over the filtration of interest, allowing additional structural
information to be encoded at multiple scales. Thus, we will explore
the extent to which substructures can be counted for different 
curvatures~(with and without a topological component) in a 
subsequent experiment, providing evidence on the expressive 
power of curvature filtrations. See \cref{tab:Substructures} for our
experimental results, and \cref{app:sec:CountingSubstructs} for a more
detailed discussion on the tendencies of each curvature notion when 
counting substructures.

\subsection{Choosing a Curvature Notion in Practice} \label{sec:ChoosingCurvature}
%
In \cref{sec:Stability}, we showed that all three of our prototypical 
curvature notions exhibit advantageous stability properties, thus
implying that they may all be used \emph{reliably} within our method to measure
the distance between two sets of graphs. However, which curvature should
be chosen in practice?
In general, we find that the answer to this question lies at the
intersection of \emph{expressivity} and \emph{scalability}, but depends
ultimately on the nature of the experiment at hand. Nevertheless, we aim
to provide an intuition for all three notions, along with a general
recommendation. We hope this, in conjunction with the experimental and
computational complexity results in \cref{sec:Experiments} and
\cref{app:sec:Computational Complexity} will help practitioners in
making an optimal choice.

\paragraph{Comparison.}
%
Forman--Ricci is arguably the simplest and most local notion of curvature. Though
limited in expressivity when compared to Ollivier--Ricci curvature, it is
by far the most efficient to compute. Resistance curvature, by contrast, is
the most global notion, making it sensitive to large substructures. However,
computing the effective resistance metric on a graph requires inverting
the Laplacian, making it far less efficient than Forman and
Ollivier--Ricci, especially for large graphs.
Finally, we have Ollivier--Ricci curvature, which we have found to be the most
expressive based on its ability to
\begin{inparaenum}[(i)]
  \item distinguish strongly-regular graphs, and
  \item count substructures. 
\end{inparaenum}
It is also the most versatile, given the option to adapt the underlying
probability measure; this comes at the cost of lower computational
performance in comparison to Forman---Ricci curvature, though.\\[0.25\baselineskip]
\highlight{%
Given its high expressivity, as well as its overall experimental and
computational performance, we recommend using \textbf{Ollivier--Ricci}
curvature whenever feasible.
}

\section{Experiments} \label{sec:Experiments}

We have proven the theoretical stability of discrete curvature notions
under certain graph perturbations. We also illustrated their ability to
distinguish distance-regular and strongly-regular graphs.
Subsequently, we will discuss empirical experiments to evaluate these
claims and to further test the utility of our methods. 

\subsection{Distinguishing Strongly-Regular Graphs}
%
\begin{wraptable}[10]{r}{0.50\linewidth}
  \vspace{-\intextsep}
  \sisetup{
    table-format    = 1.2,
    round-mode      = places,
    round-precision = 2,
    detect-all      = true,
    detect-weight   = true,
  }%
  \centering%
  \caption{%
    Success rate~($\uparrow$) of distinguishing pairs of
    strongly-regular graphs when using either raw curvature
    values or a curvature filtration.
  }%
  \label{tab:Strongly-regular graphs}%
  \let\b\bfseries
  \resizebox{\linewidth}{!}{%
    \begin{tabular}{lSSSSS}
      \toprule
        Method & \texttt{sr16622} & \texttt{sr261034} & \texttt{sr281264} &\texttt{sr401224} \\
      \midrule
      {\FRC}              &  0.00 &  0.00 &  0.00 &  0.00 \\
      {\ORC}              &\b1.00 &  0.78 &\b1.00 &  0.00 \\
      {\REC}              &  0.00 &  0.00 &  0.00 &  0.00 \\
      \midrule
      Filtration ($\FRC$) &\b1.00 &  0.20 &  0.00 &\b0.93 \\
      Filtration ($\ORC$) &\b1.00 &\b0.89 &\b1.00 &\b0.93 \\
      Filtration ($\REC$) &\b1.00 &  0.20 &  0.00 &\b0.93 \\
      \bottomrule
    \end{tabular}%
  }%
\end{wraptable}
%
In addition to the theoretical arguments outlined, we explore the
ability of our method to distinguish strongly-regular graphs in a subset
of data sets, i.e.\ \texttt{sr16622}, \texttt{sr261034},
\texttt{sr281264}, and \texttt{sr401224}. These data sets are known to
be challenging to classify since they cannot be described in terms of
the $1$-WL test~\citep{bodnar}.
Our main goal is \emph{not} to obtain the best accuracy, but to show
how the discriminative power of discrete curvature can be improved by
using it in a filtration context. \Cref{tab:Strongly-regular graphs}
depicts the results of our classification experiment. We perform a pairwise
analysis of \emph{all graphs} in the data set, calculating distances
based on histograms of discrete curvature measurements, or based
on the bottleneck distance between persistence diagrams~(`Filtrations'). 
Subsequently, we count all non-zero distances~($> \num{1e-8}$ to correct
for precision errors).
Our main observation is that combining TDA with curvature is
always better than or equal to curvature without TDA.
Similar to our theoretical predictions, both resistance curvature and
Forman curvature fail to distinguish any of the graphs without TDA. We
therefore show the benefits from an expressivity point of view for using
discrete curvature as a filtration.
Notably, we achieve the best results with
OR curvature, which is particularly flexible since it
permits changing the underlying \emph{probability measure}.
Using a probability measure based on random
walks~(see \cref{app:OR}) takes into account higher-order
neighbourhoods and improves discriminative power~(on \texttt{sr261034}, the
pairwise success rate drops to $0.0$/$0.2$ with raw/TDA values,
respectively, if the uniform probability measure is used).

\subsection{Expressivity experiments with the BREC dataset}
We evaluate discrete curvatures and their filtrations on the BREC data set, which was recently introduced to evaluate GNN expressiveness~\citep{wang2023better}. The data set consists of different categories of graph pairs~(Basic, Regular, and Extension), which are distinguishable by 3-WL but not by 1-WL, as well as Strongly-Regular (STR) and CFI graph pairs, which are indistinguishable using 3-WL. We explore the ability of curvature filtrations to distinguish these graph pairs and compare them to  substructure counting, $S_3$ and $S_4$, which involves enumerating all 3-node/4-node substructures around nodes in combination with the WL algorithm.
These approaches, unlike discrete curvature, have limited practical applications due to their high computational complexity.
Similar to our experiments on strongly-regular graphs, we calculate Wasserstein distances based on histograms of OR curvature measurements between the pairs of graphs. Subsequently, we count all non-zero distances~($> \num{1e-8}$ to correct
for precision errors).
Our main observations from \Cref{tab:BREC graphs} are that curvature \emph{can} distinguish graphs which are 3-WL indistinguishable. Additionally, we observe improvements in success rate using the filtration on the Basic, Regular, STR and Extension graph pairs. Moreover, our curvature-based approach performs competitively and sometimes even better than $S_4$, which has been shown to be extremely effective in graph learning tasks~\citep{gsn}.
Despite its empirical prowess, $S_4$ is computationally expensive, making it an \emph{infeasible} measure in many applications. Discrete curvature and its filtrations, by contrast, scale significantly better.

\begin{table}[tbp]
\centering
\caption{
    Success rate~($\uparrow$) of distinguishing pairs of graphs in the BREC dataset when using different discrete curvatures and their filtrations.
  }
\label{tab:BREC graphs}
  \sisetup{
    table-format    = 1.2,
    round-mode      = places,
    round-precision = 2,
    detect-all      = true,
    detect-mode     = true,
  }%
  \scriptsize
  \let\b\bfseries
\begin{tabular}{lSSSSS}
      \toprule
        Method & {Basic (56)} & {Regular (50)} & {STR (50)} & {Extension (97)} & {CFI (97)} \\
      \midrule
      1-WL             & 0.00        & 0.00          & 0.00      & 0.00            & 0.00 \\
      3-WL             & \b1.00        & \b1.00          & 0.00      & \b1.00            & \b0.59 \\
       \midrule
      $S_3$               & 0.86       & 0.96         & 0.00      & 0.05           & 0.00 \\
      $S_4$               & 1.00        & 0.98         &\b1.00      & 0.84           & 0.00 \\
      \midrule
      {\ORC}          & 1.00        & 0.96         & 0.06     & 0.93           & 0.00  \\
      {\FRC} & 0.96        & 0.92          & 0.00     &     0.52       &  0.00 \\
      {\REC} &\b1.00        &\b1.00          & 0.00     &   \b1.00        &  0.04 \\
      \midrule
      Filtration ($\ORC$) &\b1.00        &\b1.00          & 0.08     &   0.95          &   0.00 \\
      Filtration ($\FRC$) & 0.98        & 0.96          & 0.04     &    0.61        &  0.00 \\
      Filtration ($\REC$) &\b1.00        &\b1.00          & 0.04     & \b1.00          &  0.04 \\
      \bottomrule
\end{tabular}
\end{table} 

\subsection{Behaviour with Respect to Perturbations}\label{sec:Perturbations}

\begin{wraptable}[14]{r}{0.5\linewidth}
  \vspace{-\intextsep}%
  \caption{%
    Pearson correlation~($\uparrow$) of measures when adding/removing
    edges.
  }%
  \label{tab:Correlation}%
  \sisetup{
    table-format            = 1.3(3),
    separate-uncertainty    = true,
    retain-zero-uncertainty = true,
    detect-all              = true,
  }%
  \centering%
  \resizebox{\linewidth}{!}{%
    \begin{tabular}{lSS}
    \toprule
    Measure             &           {Adding Edges} & {Removing Edges}\\
    \midrule          
    Laplacian           &           0.457 (0.013)  &           0.420 (0.000)\\
    Clust.\ Coeff.\     &           0.480 (0.012)  &           0.504 (0.020)\\
    Degrees             &           0.761 (0.003)  &           0.995 (0.000)\\
    \midrule                    
    $\FRC$              &           0.420 (0.000)  &           0.432 (0.003)\\
    $\ORC$              &           0.903 (0.005)  &           0.910 (0.002)\\
    $\REC$              &           0.420 (0.004)  &           0.441 (0.005)\\
    \midrule          
    Filtration ($\FRC$) &           0.571 (0.006)  &           0.996 (0.006)\\
    Filtration ($\ORC$) & \bfseries 0.997 (0.000)  & \bfseries 0.970 (0.005)\\
    Filtration ($\REC$) &           0.730 (0.005)  &           0.954 (0.008)\\
    \bottomrule
  \end{tabular}
  }
\end{wraptable}
%
To explore the behaviour of curvature descriptors under perturbations,
we analyse the correlation of our metric when adding and removing edges
in the `Community Graph' data set: we increase the fraction of edges
added or removed from $0.0$ to $0.95$, measuring the distance between
the perturbed graphs and the original graphs for each perturbation
level. Following \citet{OBray22a}, we require a suitable distance
measure to be highly correlated with increasing amounts of perturbation.
We compare to current approaches that use descriptor functions with MMD.
As \cref{tab:Correlation} shows, a curvature filtration yields a higher
correlation than curvature in combination with MMD, showing the benefits
of employing TDA from a stability perspective. Additionally, curvature
filtrations improve upon the normalized Laplacian and clustering
coefficient~(again, we used MMD for the comparison of these distributions). OR curvature exhibits particularly strong results when
adding/removing edges, even surpassing the local degrees of
a graph~(which, while being well-aligned with perturbations of the
graph structure, suffer in terms of overall expressivity).

\subsection{Counting Substructures}

The ability of a descriptor to count local substructures is important for evaluating its expressive power~\citep{local_graph_parameters}. 
We follow \citet{counting_subgraphs_gnn}, who assess the ability of GNNs to count
substructures such as triangles, chordal cycles, stars and tailed
triangles.
This is achieved by generating regular graphs with random edges removed,
and counting the number of occurrences of each substructure in a given
graph.
To assess the power of curvature to capture such local structural
features, we use the same experimental setup and pass the edge-based
curvatures through a simple $1$-layer MLP to output the substructure
count. Additionally, we evaluate the effect of using curvature as
a filtration.

\begin{wraptable}[14]{l}{0.5\linewidth}
  \centering%
  \sisetup{
    table-format         = 1.2,
    round-mode           = places,
    round-precision      = 2,
    detect-all           = true,
    detect-weight        = true,
  }%
  \caption{%
    MAE~$(\downarrow)$ for counting substructures based on raw curvature
    values and curvature filtrations.
    The `Trivial Predictor' always outputs the mean training target.
  }%
  \label{tab:Substructures}
  \resizebox{\linewidth}{!}{%
    \begin{tabular}{lSSSS}
      \toprule
    Method  & {Triangle}        & {Tailed Tri.}     & {Star}            & {$4$-Cycle}         \\
    \midrule
    Trivial Predictor & 0.878829 & 0.897183  &  0.813312  &  0.92724 \\
    GCN                                  & 0.4186          & 0.3248
                                         & \bfseries 0.1798
                                         & \bfseries 0.2822         \\
    \midrule
    \FRC                & 0.5383          & 0.5575          & 0.7217          & 0.5259          \\
    \ORC                & 0.3305          & 0.3080          & 0.3978          & 0.3145          \\
    \REC                & 0.5859          & 0.4990          & 0.7222          & 0.4707          \\
    \midrule
    Filtration (\FRC)   &    0.4495       &    0.5170       &    0.4851       &      0.6021      \\ 
    Filtration (\ORC)   & \bfseries 0.2321         & \bfseries 0.2395          & 0.3393          & 0.3089           \\ 
    Filtration (\REC)   &     0.4745      &    0.4777      &     0.3571      &      0.4213      \\ 
      \bottomrule
    \end{tabular}%
  }%
\end{wraptable}
%
\Cref{tab:Substructures} shows our experimental results, reported on the
pre-defined test split defined by \citet{counting_subgraphs_gnn}.
In comparison to Graph Convolutional Networks~\citep{kipf}, 
we find that OR curvature exhibits improved performance in counting
small local structures such as triangles and tailed triangles; this is
a surprising finding given the smaller computational footprint of this
curvature formulation~(in comparison to a GCN). Moreover, we find that
the OR curvature performs better than both the Forman curvature and the
resistance curvature for the different substructures in the data. We
also observe that combining curvature with TDA almost always improves
upon using the curvature alone, the only exception being Forman
curvature for counting \mbox{$4$-cycles}.
We leave a more detailed investigation of these phenomena for future work.

\subsection{Synthetic Graph Generative Model Evaluation}
 
To have a ground truth for a graph distribution, we tested our
metric's ability to \emph{distinguish graphons}. 
Following the approach suggested by
\citet{graphon_experiment}, we generate four graphons,
$W_1(u, v) = uv$, $W_2(u, v) = \exp\{-\max(u, v)^{0.75}\}$, $W_3(u, v)
= \exp \{-0.5 * ( \min(u, v) + u^{0.5} + v^{0.5})\}$ and $W_4(u, v) = \|u
- v \|$. Sampling from these graphons produces dense graphs, and we
control their size to be between $9$ and $37$ nodes, thus ensuring
that we match the sizes of molecular graphs in the \texttt{ZINC}
data set~\citep{zinc}, an important application for generative models. 
  
We perform experiments by considering all combinations of three and four
graphons. We generate distances between graphs with our method as well as other kernel-based approaches, and use
spectral clustering to separate the distributions. We measure the
performance of the algorithms using the Adjusted Rand Index~(ARI) of the
predicted clusters, comparing to three state of the art, kernel-based
approaches:
\begin{inparaenum}[(i)]
  \item Wasserstein Weisfeiler--Le(h)man graph kernels~\citep{wwlgk},
  \item graph neural tangent kernels~\citep{gntk}, and
  \item Tree Mover's Distance~\cite{chuang2022tree}.
\end{inparaenum}
From \Cref{fig:1}, we find that a filtration based on OR curvature is better able to
distinguish and cluster graphons than the previously-described approaches 
based on kernels and it performs best for all sets of graphons.
We also observe that OR curvature performs better than other discrete
curvatures, with resistance curvature achieving higher ARI than Forman
curvature.
Notice that unlike these kernel approaches, our method can be easily
extended to provide a proper metric between distributions of graphs. 

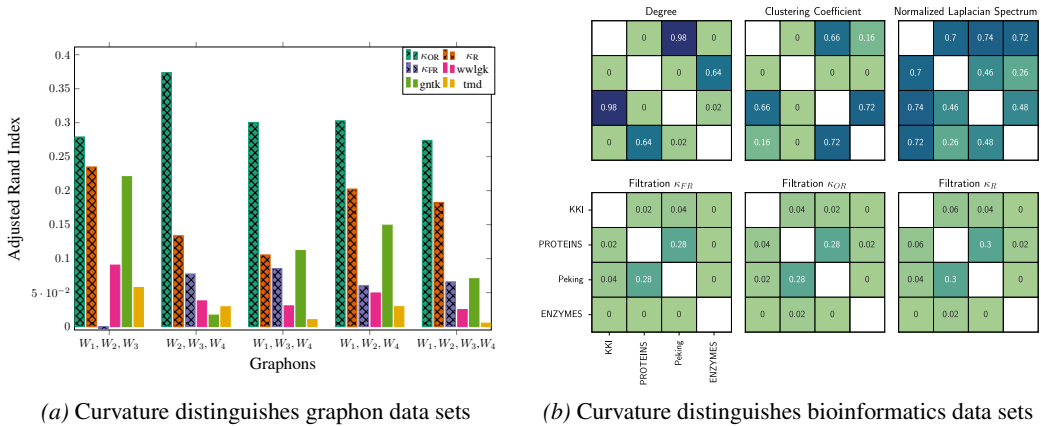
\begin{figure}[tbp]
  \centering
  \subcaptionbox{%
    Curvature distinguishes graphon data sets\label{fig:1}%
  }{%
    \pgfplotstableread[col sep=comma]{data/graphon_scores_kernel.csv}\table

\resizebox{0.475\linewidth}{!}{%
\begin{tikzpicture}
\begin{axis}[
    ybar,
    bar width=.28cm,
    width=\textwidth,
    height=10cm,
    ymin = -0.005,
    xtick = data,
    xticklabels ={{$W_1,W_2,W_3$},{$W_2,W_3,W_4$},{$W_1,W_3,W_4$}, {$W_1,W_2,W_4$},{$W_1,W_2,W_3$,$W_4$}},
    x tick label style={font=\normalsize,yshift=0cm,align=center},
    ylabel={\Large Adjusted Rand Index},
    legend style={area legend, at={(0.898,1)}, anchor=north,legend columns=2},
    xlabel=\Large Graphons,
    every axis plot/.append style={fill},
    cycle list name=Dark2-8
    ]
 \addplot +[postaction = {pattern=crosshatch}] table [x expr=\coordindex, y=or] {\table};
        \addlegendentry{$\ORC$}
    \addplot+[postaction = {pattern=crosshatch}] table [x expr=\coordindex, y=resistance] {\table};
        \addlegendentry{$\REC$}
    \addplot+[postaction = {pattern=crosshatch}] table [x expr=\coordindex, y=forman] {\table};
        \addlegendentry{$\FRC$}
    \addplot table [x expr=\coordindex, y=wwlgk] {\table};
        \addlegendentry{wwlgk}
     \addplot table [x expr=\coordindex, y=gntk] {\table};
    \addlegendentry{gntk}
    \addplot table [x expr=\coordindex, y=TMD] {\table};
    \addlegendentry{tmd}
        \legend{$\ORC$,$\REC$,$\FRC$,wwlgk, gntk, tmd}
\end{axis}
\end{tikzpicture}
}%
    \vspace{0.1cm}
  }\quad%
  \subcaptionbox{%
    Curvature distinguishes bioinformatics data sets\label{fig:2}%
  }%
  {%
\includegraphics[width=0.475\linewidth, scale=0.28]{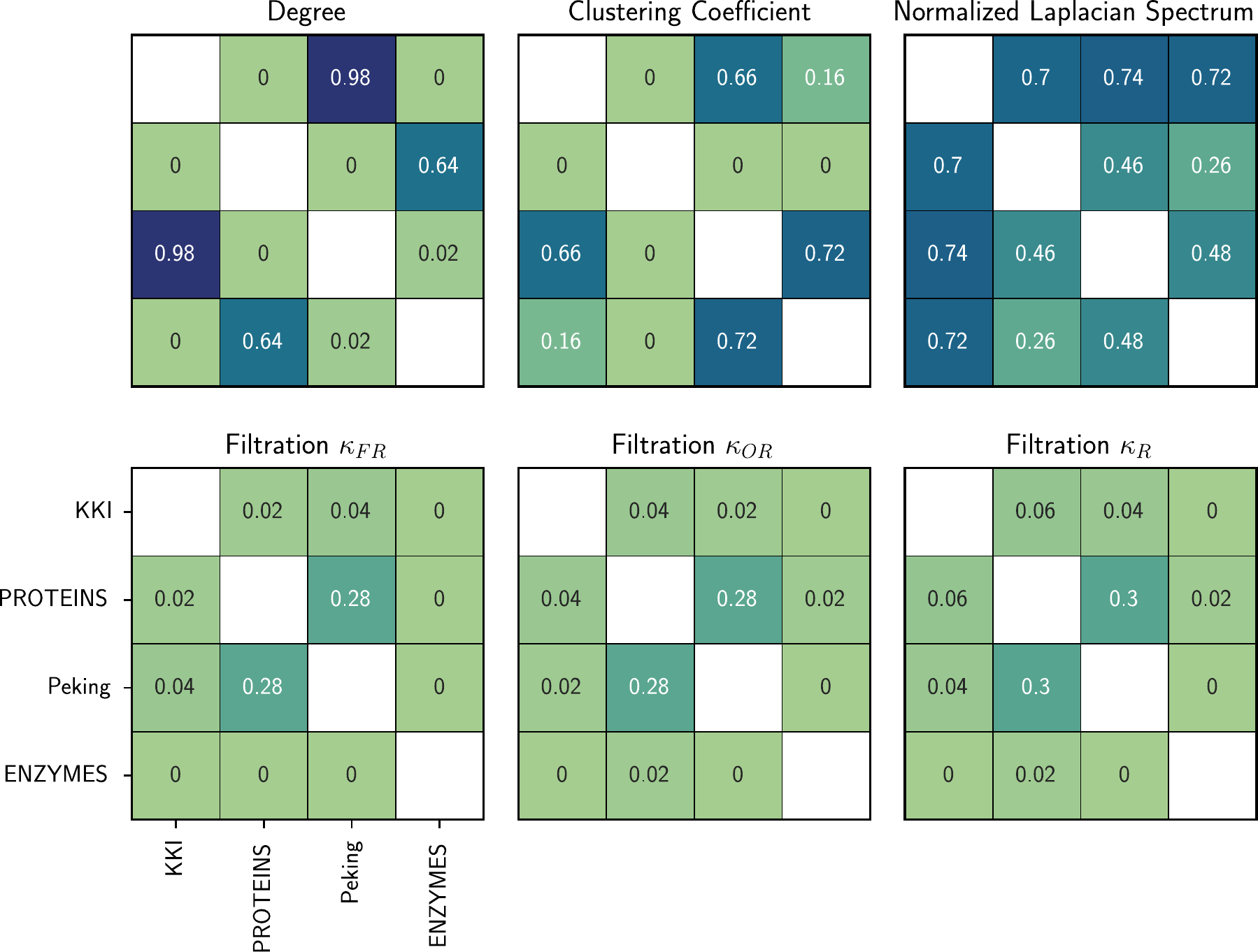}
  }%
  \caption{%
    \subref{fig:1} Adjusted Rand Index~($\uparrow$) for clustering sets of four
    graphons. We compare our curvature filtrations (\legendbox{black})
    to kernel-based methods.
    \subref{fig:2} Permutation testing values~$(\downarrow)$ for distinguishing
    different bioinformatics data sets. Position~($i, j$) in each matrix
    denotes a permutation test between data set~$i$ and data set~$j$.
  }
\end{figure}

\subsection{Real-World Graph Generative Model Evaluation}

By converting the generated persistence diagrams into
a \emph{persistence landscape}, we can generate an average topological
descriptor for all graphs in
a distribution~\citep{Bubenik15}.
%
This allows us to calculate norms between graph distributions, making it
possible to perform two-sample and permutation testing, unlike
a majority of kernel-based approaches, which provide a distance between
\emph{all} individual graphs, or would require MMD to assess mean
similarities.
We randomly sample ten graphs from four different bioinformatics
data sets, i.e.\ KKI, PROTEINS, Peking, and ENZYMES~\citep{tudatasets}.
We measure the distance between two data sets using the $L^P$ norm
between their average persistence landscapes. We then permute graphs from both samples, 
randomly selecting sets of equal size, measure their respective distances, and
finally aggregate the fraction of distances which are higher than the original. 
A low fraction indicates that distances are \emph{lower} for permutations 
than between the original sets of graphs, suggesting that the metric can 
distinguish between the two distributions. We compare our approach to previous 
methods that combine graph statistics with MMD.
Using a significance threshold of $p < 0.05$, we see in \Cref{fig:2}
that both Forman and the OR curvature are able to distinguish \emph{all but
one pair of data sets, an improvement over all the other approaches}. In
general, we find that fractions are lower using curvature filtrations
than graph statistic based approaches, demonstrating the utility of our
approach.

\section{Conclusion}

We have described the first thorough analysis
of both \emph{stability} and \emph{expressivity} of discrete curvature
notions and their filtration formulations on graphs.
We believe this to be important for the community in multiple contexts, ranging from
improving expressivity of GNNs to understanding the robustness of curvature on graphs.
Using curvature filtrations and their topological
descriptors~(here: \emph{persistence landscapes}), we develop a new metric to measure
distances between graph distributions.
Our metric can be used for evaluating graph generative models, is robust
and expressive, and improves upon current approaches based on graph
descriptor and evaluator functions.
We have also demonstrated clear advantages over state-of-the-art
methods that combine graph statistics with MMD, providing instead 
a metric with
\begin{inparaenum}[(i)]
 \item well-understood parameter and function choices,
 \item stability guarantees,
 \item added expressivity, and
 \item improved computational performance.
\end{inparaenum}
Most notably, we scale \emph{significantly} better for large populations
of molecular-sized graphs~(see \cref{app:sec:Computational Complexity}
for more details), which we consider crucial for current graph generative
model applications.
We hope that our pipeline will provide a principled, interpretable,
and scalable method for practitioners to use when evaluating the
performance of graph generative models. 

Future work could explore other representations~\citep{Adams17a,
Rieck20c}, focus on different filtration
constructions~\citep{Cohen-Steiner09}, new curvature measures, or
further extend our stability and expressivity results~(for instance to
the setting of weighted graphs).
Given the beneficial performance and flexibility of Ollivier--Ricci curvature in our experiments, we believe that changing---or \emph{learning}---the probability measure used in its calculation could lead to further improvements in terms of expressivity, for example.
Another relevant direction involves incorporating node and edge
features into the distance measure and applying the model to specific
use cases such as evaluating molecule generation. We envision that this
could be done using bi-filtrations or by considering node features
for the curvature calculations.

\begin{ack}
The authors want to thank the anonymous reviewers for their comments, which helped us improve the paper.
We are also grateful for the area chair who believed in this work.
The authors are grateful to Fabrizio Frasca and Leslie O'Bray for valuable feedback on early versions of the manuscript.
J.S. is supported by the UKRI CDT in AI for Healthcare http://ai4health.io (Grant No. P/S023283/1).
M.B. is supported in part by the ERC Consolidator grant No. 724228 (LEMAN) and the EPSRC Turing AI World-Leading Researcher Fellowship.
B.R.\ is supported by the Bavarian State government with
funds from the \emph{Hightech Agenda Bavaria}.
He wishes to dedicate this paper to his son Andrin.
\end{ack}

\bibliography{main}
\bibliographystyle{plainnat}

\clearpage

\appendix
\onecolumn

\section{Pseudocode}\label{app:pseudocode}
Here we give pseudocode for various parts of the method, highlighting the most relevant aspects of using curvature filtrations to evaluate graph generative models. First, we outline a crucial part of our evaluation framework in \cref{alg:compute_average}: how we compute summary topological descriptors for sets of graphs. This algorithm, taken from \citet{Bubenik15}, assumes a list of precomputed persistence landscapes, one for each graph. \cref{alg:compare_graph_sets}, on the other hand, outlines our procedure for generating a distance between two sets of graphs using their summary topological descriptors.

\begin{algorithm}
\caption{Compute Average of Persistence Landscapes}
\label{alg:compute_average}
\begin{algorithmic}
\Require{$\Lambda$ is a list of persistence landscapes}
\Ensure{$\Bar{\Lambda}$ is the average persistence landscape}
\Function{ComputeAverageLandscape}{$\Lambda$}
    \State $n \gets$ length of $\Lambda$
    \State $D \gets$ maximum homology degree occurring in $\Lambda$
    \State $\Bar{\Lambda} \gets$ empty persistence landscape with $D$ homology dimensions.
    \State $\lambda^i(k,t) \gets$ the piecewise-linear function encoding hom-deg $k$ contained in $\lambda^i \in \Lambda$.  
    \State $R\gets$ maximal domain for each function contained in $\Lambda$.
    
    \State $k \gets 0$ 
\While{$k \leq D$ }
    \For{$t \in R$}  
       \State $\Bar{\Lambda}(k, t) \gets \frac{1}{n} \sum_i^n \lambda^i(k,t)$
        \EndFor
    \State $k \gets k+1$
\EndWhile
\State \textbf{return} $\Bar{\Lambda}$
\EndFunction
\end{algorithmic}
\end{algorithm}

Recall that persistence landscapes are \emph{a collection of piecewise linear functions} that encode the homological information tracked by a specified filtration. Thus, to compute the average landscape $\Bar{\Lambda}$ at each dimension, we simply sum the piecewise linear functions, and divide by the total number of landscapes in consideration. To understand our method for comparing sets of graphs, and thus evaluating graph generative models, based on average persistence landscapes see \cref{alg:compare_graph_sets}.

\begin{algorithm}
\caption{Computing Distance between Two Sets of Graphs}
\label{alg:compare_graph_sets}
\begin{algorithmic}
\Require $\mathcal{G}_1$ is a set of graphs.
\Require $\mathcal{G}_2$ is a set of graphs.
\Require $D$ to be the maximum homology degree to be computed.
\Require $f$ is a function to generate a filtration for persistent homology computations.
\Ensure{$Dist$} is the distance between $\mathcal{G}_1,\mathcal{G}_2$. 

\vspace{0.4cm}
\Function{ComputeSetDistance}{$\mathcal{G}_1,\mathcal{G}_2$}
\State $\Lambda_1 \gets$ is an empty list.
\State $\Lambda_2 \gets$ is an empty list.
\State $DistVec \gets$ is an empty $D+1$-dimensional vector. 
\For{$G\in \mathcal{G}_1$}
    \State $L\gets \text{ComputePersistenceLandscape}(G,f)$
    \State $AddItem(\Lambda_1,L)$
\EndFor

\For{$G\in \mathcal{G}_2$}
    \State $L\gets \text{ComputePersistenceLandscape}(G,f)$
    \State $AddItem(\Lambda_2,L)$
\EndFor

\State $\Bar{\Lambda_1} \gets \text{ComputeAverageLandscape}(\Lambda_1)$ 
\State $\Bar{\Lambda_2} \gets \text{ComputeAverageLandscape}(\Lambda_1)$ \Comment{By \cref{alg:compute_average}.}

\State $k \gets 0$ 
\While{$k \leq D$ }
    \State  $DistVec[k] \gets |supNorm(\Bar{\Lambda_1}(k,t)) -  supNorm(\Bar{\Lambda_2}(k,t))|$
\EndWhile

\State $Dist \gets ||DistVec||_2$
\State \textbf{return} $Dist$

\EndFunction

\vspace{0.4cm}
\Function{ComputePersistenceLandscape}{$G,f,D$}
    \State $P \gets$ compute persistence diagram for $G$ using $f$ for homology degree $k\in \{0,\cdots, D\}$. 
    \State $L \gets$ transform $P$ into a persistence landscape. \Comment{See \citep{Bubenik15} for implementation.} 
    \State \textbf{return} $L$
\EndFunction
\end{algorithmic}
\end{algorithm}

\section{Additional Proofs on Stability and Expressivity}\label{app:Proofs}

\subsection{Stability Proofs}

\thmstabilityupperbound*

\begin{proof}
  Considering the calculation of persistence diagrams based on
  scalar-valued filtrations functions, every point in the persistence
  diagram~$D_f$ can be written as a tuple of the form~$(f(e_F), f(e_F'))$,
  with $e_F, e_F' \in E_F$; the sample applies for $D_g$. The inner distance
  between such tuples that occur in the bottleneck distance calculation
  can thus be written as
  \begin{equation}
    \|(f(e_F), f(e_F')) - (g(e_G), g(e_G'))\|_\infty.
  \end{equation}
  The maximum distance that can be achieved using this expression is
  determined by the maximum variation of the functions, expressed via
  $\mathrm{dis}(f, g)$ and $\mathrm{dis}(g, f)$, respectively.
\end{proof}

\paragraph{Edge Filtrations versus Vertex Filtrations} \label{app:sec:edge_v_vertex_filtrations}
Our results are structured to address filtrations built by a function on the \emph{edges} of a graph
$G = (V,E)$, $f\colon E \to \reals$. This matches our notions of
discrete curvature, which are also defined edge-wise. $~f$ gives an explicit ordering on $E$ and thus
an induced ordering on $V$ given by:
$$ v \leq v' \iff  \sum_{e\in E_v} f(e) \leq \sum_{e'\in E_{v'}} f(e')$$

where $E_x$ is the set of edges incident to $x\in V$. However, one can also define a filtration directly over \emph{vertices} with a scalar valued function $h\colon V \to \reals$. By assumption,~$h$ can attain only a finite number of values~, call the unique values $b_1, b_2, \dots b_k$. Thus, we can also compute a filtration
$
\emptyset \subseteq G_0 \subseteq G_1 \ldots \subseteq G_{k-1} \subseteq G_k = G,
$
where each $G_i := (V_i, E_i)$, with $V_i := \{ v \in V \mid h(v) \leq
b_i \}$ and $E_i := \{ e \in E \mid \max_{v \in e} h(v) \leq b_i \}$. Similarly, the
explicit ordering on $V$ given by$~h$ induces an ordering on $E$:
$$ e \leq e' \iff  \max_{v\in e} h(v) \leq \max_{v' \in e'} h(v')$$

The key idea here is that \emph{either choice gives rise to an ordering of both edges, and vertices}
that are used to calculate persistent homology of the graph. This means that the arguments for \cref{thm:Stability}
and \cref{thm:Expressivity} also bound the bottle-neck distance for persistence diagrams generated
using vertex filtrations.

\paragraph{Graph perturbations.}
Here we explicitly specify a common framework used in the proofs for
stability of curvature functions. As mentioned in the main text, we
consider perturbations to \emph{unweighted}, \emph{connected} graphs
$G=(V,E)$, with $|V|=n$ and $|E|=m$. In the case of \emph{edge
addition}, let $i*$ and $j*$ be arbitrary vertices that we wish to
connect with a \emph{new} edge, forming our new graph $G' = (V,E')$
where $E' = E \cup (i*,j*)$ such that $|E'| = m+1$. For \emph{edge
deletion}, we similarly let $(i*,j*)\in E$ be the edge we delete such
that $E'\subset E$ and $|E'|=m-1$. Moreover, we only consider edges
$(i*,j*)$ that leave $G'$ connected. 

\subsubsection{Forman--Ricci Curvature}

\thmFormanAdditionAndDeletion*

\begin{proof}
  We first handle the case of \textbf{edge addition}, using the graph perturbation framework specified above.
 By definition $\FRC(i,j)$ depends \emph{only} on the degrees of the source and target $(i,j)\in E$ and the number of triangles formed using
$(i,j)$, $|\#_{\Delta_{ij}}|= |N(i)\cap N(j)|$, where $N(i)$, $N(j)$ are the set of neighbouring nodes for $i,j$ respectively. This
is a local computation– all relevant information can be computed in the
subgraph surrounding the inserted edge $(i*,j*)$. Thus, in order to understand stability of $\FRC(i,j)$, we need to understand how $N(i)$ and $N(j)$ change under graph perturbations. For our new graph $G'$, the only edges with potential to change their curvature lie in the set:

$$E_{*} := \{(u,v) \in E| u,v \in N(i*)\cup N(j*)\}$$

For the new edge $(i*,j*)\in E_{*}$, we can directly compute $\FRC(i*,j*)$ based on the original structure of the graph. However, in terms of stability we are interested in the other members of $E_{*}$, i.e. edges in the original graph. The analysis of $E_{*}$ can be split into two cases: one of the nodes is $i*$ or $j*$ or neither is. 

\emph{Case 1}: WLOG assume $(i,j) = (i*,j)\in E_{*}$.
Clearly, $d'_{i} = d_{i}+1$. As for $|\#'_{\Delta_{ij}}|$, this can maximally be increased by 1 in the case that $j*\in N(j)$, else the triangle count stays the same. 

\emph{Case 2}: Let $(i,j)\in E_{*}$ where $i,j \in V\setminus \{i*,j*\}$. In this case, there is no change to the degree nor the number $|\#'_{\Delta_{uv}}|$. 

Thus \emph{Case 1} defines the bounds which are dictated as follows: if $(i*,j*)$ forms a new triangle, our curvature can increase by 2, and if no triangle is formed the curvature can decrease by 1 in response to the increased degree. Thus we can bound $\FRC'(i,j) := 4 - d'_i-d'_j + 3|\#'_{\Delta_{ij}}|$ as follows:
$$ \FRC(i,j) -1 \leq \FRC'(i,j) \leq \FRC(i,j) +2 $$

The case of \textbf{edge deletion} can be handled similarly.
Again, we need only consider the edges in $E_{*}$, as defined in the proof above, and can make the same case argument. 

\emph{Case 1} WLOG assume $(i,j) = (i*,j)\in E_{*}$.
Clearly, $d'_i = d_i-1$. As for $|\#'_{\Delta_{ij}}|$, this can maximally be decreased by 1. 

\emph{Case 2}:  Let $(i,j)\in E_{*}$ where $i,j \in V\setminus \{i*,j*\}$. Degree and number of triangles do not change in response to the perturbation.

Again, \emph{Case 1} gives rise to the following bounds hold for $\FRC'$:

$$ \FRC(i,j) -2 \leq \FRC'(i,j) \leq \FRC(i,j)+1 $$

\end{proof}

\subsubsection{Ollivier--Ricci Curvature}

The definition of $\ORC$ establishes a relationship between the graph metric $d_G$, the Wasserstein distance $W_1$, the probability distributions $\mu_i,\mu_j$ at nodes $i,j$ and the curvature. Given that we are considering unweighted, and connected graphs we know that $(V,d_G)$ is a well-defined metric space and therefore $W_1$ (as defined in \cite{OLLIVIER2007643}) defines the $L_1$ transportation distance between two probability measures $\mu_i,\mu_j$ with respect to the metric $d_G$. This is relevant for a much larger class of graph metrics than just the standard choice of the shortest path distance. We use results from \cite{OLLIVIER2007643} and the metric properties of $W_1$ and $d_G$ on graphs to bound the potential changes in $\ORC$ following an edge perturbation. 

\begin{lemma}{\label{ORlemma}}
  Consider the triple $\mathcal{G}= (G, d_G,\mu)$. Let $\delta_i$ denote the Dirac measure at node $i$ and $J(i)$ :=
  $W_1(\delta_i, \mu_i)$ the corresponding jump probability in the graph
  $G$. The Ollivier--Ricci curvature $\ORC(i,j)$ satisfies the following
  Bonnet-Myers inspired upper bound:
  \begin{equation}\label{or_jump_bound}
    \ORC(i, j) \leq \frac{J(i)+J(j)}{d_G(i,j)}
  \end{equation}
\end{lemma}
\begin{proof}
Rearranging the original definition for OR curvature gives:
$$
    W_{1}(\mu_{i}, \mu_{j}) = d_G(i, j)(1 - \ORC(i, j))
$$
By definition of the $W_1$, we have $d_G(i,j) = W_1(\delta_i,\delta_j)$. Using this and the fact that $W_1$ satisfies the triangle inequality property, we can construct the desired upper bound on $\ORC$:  
\begin{align*}
d_G(i,j) \leq& W_1(\delta_i,\mu_i) + W_1(\mu_i,\mu_j) + W_1(\delta_j,\mu_j) \\
d_G(i,j) \leq& J(i) + d_G(i, j)(1 - \ORC(i, j)) + J(j) \\
d_G(i,j)(1- (1 - \ORC(i, j))) \leq& J(i) + J(j) \\
 \ORC(i, j) &\leq \frac{J(i)+J(j)}{d_G(i,j)}
\end{align*}
\end{proof}

\thmORBonnetMyers*
\begin{proof}
  We first prove the \emph{upper bound}.
  Given that $G'$ is still connected (by assumption), and both $W'_1$ and $d_{G'}$ still
  satisfy the metric axioms, this result follows directly from
  \cref{ORlemma}.
  For proving the \emph{lower bound}, recall from \cref{sec:Stability}
  that $\mathcal{G}' = (G',d_{G'},\mu')$ specifies the behaviour of the
  new graph metric $d_{G'}$ and the and the updated probability measure
  $\mu'$ in response to the perturbation. Moreover, this defines a new
  Wasserstein distance $W'_1$ and we will show that the maximum reaction
  (as evaluated by $W'_1$) to the perturbation $W'_{\max} := \max_{x\in
  V} W'_{1}(\mu'_{x},\mu_x)$ can be used to express a general lower
  bound for OR curvature in the event of a perturbation. As per
  \cref{eq:Ollivier--Ricci}, we can define our curvature following the
  perturbation as:
  \begin{equation}
  \ORC'(i, j) = 1 - \frac{1}{d_{G'}(i, j)}W'_{1}(\mu'_{i}, \mu'_{j}) 
  \end{equation} 
  Once again, we can make use of the metric properties of $W'_{1}$, to
  establish the lower bound as
  \begin{align*}
      \ORC'(i, j) &\geq 1 - \frac{1}{d_{G'}(i, j)}\big[W'_{1}(\mu_{i}, \mu'_{i})+ W'_{1}(\mu_{j}, \mu'_{j})+W'_{1}(\mu_{i}, \mu_{j})\big] \\
      &\geq \frac{1}{d_{G'}(i, j)}\big[2W'_{\max}+W'_{1}(\mu_{i},
      \mu_{j})\big].
  \end{align*}
\end{proof}

\subsubsection{Resistance Curvature} \label{app:Resistance}
\textbf{A brief clarification on inverting edge weights.} The common practice when computing effective resistance is to invert the edge weights of a graph in order to get a resistance. Given the spirit of resistance from circuit theory, we know that a high resistance should make it difficult for current to pass between nodes. Analogously when thinking about our graph as a markov chain, this would correspond to a low transition probability. So, if we think about our edge weights as coming from some kernel where higher similarity results in a higher edge weight, then we should definitely invert our edge weights to get to resistance. However, in the case that our edge weights represent the cost of travelling between nodes, then this is a suitable proxy for resistance in which case inverting the nodes is unnecessary. In order to achieve the theoretical properties of curvature with well known examples described in \cite{Devriendt22}, we \textit{do not} invert the edge weights in our experiments. Which means that the curvature itself interprets the edge weights themselves as a cost/resistance; we think it is an important point to specify especially given the similarity to markov chains and the borrowed terminology from circuit theory.

\textbf{Definitions.} The resistance distance, intuitively, measures how well connected two nodes are in a graph. It is defined in \cite{Devriendt22} as:
\begin{equation}\label{eq:resistance_distance}
    R_{ij} := (\textbf{e}_i-\textbf{e}_j)^\intercal Q^\dagger (\textbf{e}_i-\textbf{e}_j)
\end{equation}
Here $Q$ is the normalized laplacian (weighted degrees on the diagonal, see \cite{Devriendt22}), $Q^\dagger$ the Moore-Penrose inverse, and $\textbf{e}_i$ is $i^{th}$ unit vector. This is the main feature that will be studied to understand the stability of the curvature measure, and  can be computed for any two nodes in a connected component of a graph.

 Recalling the equations for node resistance curvature and resistance curvature, i.e.\ \cref{eq:Resistance}, it becomes clear that the main task is to understand \emph{how the resistance distance changes in response to perturbations}. The results below from \cite{Lovasz96}, are crucial for our proofs. Let $C(i,j)$ be the commute time between nodes $i,j\in V$. It is important to note that these results depend on the \emph{normalized Laplacian}, defined in \cite{Lovasz96} as $N = D^\frac{1}{2} A D^\frac{1}{2}$, with eigenvalues $\lambda_i$, ordered such that $\lambda_1 \geq \lambda_2 \geq ...$. Here, $D$ is the diagonal matrix with inverse degrees and $A$ the adjancecy matrix. Also, as is consistent with the rest of the paper, assume our graph has $n$ nodes and $m$ edges, and $d_i$ is the degree at node $i\in V$.

\begin{proposition}
For a graph $G$, let $N= D^\frac{1}{2} A D^\frac{1}{2}$ be the \emph{normalized Laplacian} with eigen values $\lambda_1 \geq \lambda_2 \geq \cdots \geq \lambda_n$. Then, the commute time in $G$ between nodes $i,j$ is subject to the following bounds:
\begin{equation}\label{eq:Lovasz3.3}
m\big(\frac{1}{d_s}+\frac{1}{d_t}\big) \leq C(i,j) \leq \frac{2m}{1-\lambda_2}\big(\frac{1}{d_s}+\frac{1}{d_t}\big)
\end{equation}
\end{proposition}
\begin{proposition}
Consider the unweighted graph G, where each edge represents a \emph{unit resistance}, i.e we consider each edge in the graph to be artificially weighted with value 1. Then the following equality holds for the commute time between nodes $i,j$:

\begin{equation} \label{eq:Lovasz4.1}
    C(i,j) = 2m R_{ij} 
\end{equation}
\end{proposition}

\begin{proposition}
    If $G'$ arises from a graph G by adding a new edge, then the commute time $C'(i,j)$ between any two nodes in $G'$ is bounded by:
    \begin{equation}\label{eq:Lovasz2.9} 
         C'(i,j) \leq (1+\frac{1}{m})C(i,j) 
    \end{equation}
\end{proposition}

For proofs of these propositions, we refer the reader to \cite{Lovasz96}. These results create a direct connection between commute times and resistance distance, and gives insight into how commute time reacts under edge addition, and we use them directly to generate our bounds for resistance curvature.

\thmResistanceAddition* 
\begin{proof}
    Let $R'_{ij}$ be the resistance distance in $G'$. Likewise, let $C(i,j)$ be the commute distance in $G$ between nodes $i,j$ and $C'(i,j)$ be the commute time in $G'$. Then \cref{eq:Lovasz4.1} and \cref{eq:Lovasz2.9} ensure that $R'_{ij}$ is bounded above, by the original resistance distance in $G$:
\begin{align*}
    2(m+1)R'_{ij} &\leq 2m (1+\frac{1}{m}) R_{ij} \\
    R'_{ij} &\leq R_{ij}
\end{align*}

This follows our intuition of resistance distance very well: with the addition of an edge nodes can only get more connected. \cref{eq:Lovasz3.3} also gives a nice lower bound:
\begin{align*}
    (m+1)\big(\frac{1}{d'_i}+\frac{1}{d'_j}\big) &\leq C'(i,j) \\
    \frac{1}{2}\big(\frac{1}{d'_i}+\frac{1}{d'_j}\big) & \leq R'_{ij}
\end{align*}

In the case that we are adding a single edge, it is often the case that node degrees remain constant. However, the nodes that are connected by the new edge, $(i*,j*)\in E'\setminus E$, increase such that $d'_{i*} = d_{i*}+1$ and $d'_{j*} = d_{j*}+1$. Thus, the following lower bound holds in general for $R'_{ij}$ and we can remain agnostic to the precise location of the new edge:

\begin{equation}
     \frac{1}{2}\big(\frac{1}{d_i+1}+\frac{1}{d_j+1}\big) \leq R'_{ij} \leq R_{ij}
\end{equation}

And likewise, after adding $p$ edges:

$$    \frac{1}{2}\big(\frac{1}{d_i+p}+\frac{1}{d_j+p}\big) \leq R^p_{ij} \leq R_{ij} $$

So the bounds of our \emph{perturbed} resistance distance $R'_{ij}$ are determined by the initial network structure ($R_{ij}$) and the number connections each specific vertex has. Naturally, certain node pairs will be more strongly affected by the addition of an edge. We can define the maximum reaction to perturbation across pairs as follows:

\begin{equation}
     \Delta_{add} := \max_{i,j\in V}\bigg(R_{ij}- \frac{1}{2}\big(\frac{1}{d_i+1}+\frac{1}{d_j+1}\big)\bigg)
\end{equation} 

This can be used to bound node resistance curvature. In an unweighted graph, we have

$$p_i = 1 - \frac{1}{2}\sum_{j\sim i} R_{ij}$$

$$p'_i = 1 - \frac{1}{2} \sum_{j\sim i} R'_{ij}$$
For $G$ and $G'$ respectively. Given that resistance can only increase, $p_i$ is clearly an lower bound for $p'_i$. Certainly an upper bound occurs when when the resistance between each one of i's neighbors maximally decreases. Thus we get the following inequality:

\begin{equation}
    p_i \leq p'_i \leq p_i + \frac{d_i}{2}\Delta_{add}      
\end{equation}
 
Finally this gives the desired bound on $\REC'$:
$$
\REC(i,j) \leq \REC'(i,j) \leq \REC(i,j) + \frac{\Delta_{add}(d_i+d_j)}{R_{ij}-\Delta_{add}}
$$
\end{proof}

\begin{restatable}{theorem}{thmResistanceDeletion}\label{thm:ResistanceDeletion}
If $G'$ is the graph generated by \textbf{edge deletion}, then $\REC'\leq \REC$, bounded by:
\begin{align*}
|\REC'(i,j) - \REC(i,j)| \leq \frac{1}{R_{ij}+\Delta_{del}}
  \Big[\frac{2}{R_{ij}}(2R_{ij}+\Delta_{del})(p_i+p_j) - \Delta_{del}(d_i+d_j)\Big],
\end{align*}
where $\Delta_{del} = \frac{2}{1-\lambda_2} - \min_{i,j\in V}(R_{ij})$ and $\lambda_2$ is the second largest eigenvalue of $N$.
\end{restatable}

\begin{proof}
    Now we can beg the question of how effective resistance changes when we remove an edge. By inverting our initial argument in above proof of \cref{thm:ResistanceAddition}, we know that after removing an edge our resistance distance can only increase. Formally, $R_{ij} \leq R'_{ij}$. For the upper bound, we can once again make an argument using \cref{eq:Lovasz3.3}, this time relying on the other half of the inequality. Here we need to also mention the normalized Laplacian $N'$ for $G'$, with eigenvalues $\lambda'_1\geq \lambda'_2\geq... \geq \lambda'_n$.

\begin{align*}
    C'(i,j) &\leq \frac{2(m-1)}{1-\lambda'_2}\big(\frac{1}{d'_i} + \frac{1}{d'_j} \big) \\
    R'_{ij} &\leq \frac{1}{1-\lambda'_2}\big(\frac{1}{d'_i} + \frac{1}{d'_j} \big)
\end{align*}

Again, we know that only the two unique vertices ($i*,j*)$ that shared an edge will have affected degrees, s.t $d'_{i*} = d_{i*}-1$ and $d'_{j*} = d_{j*}-1$. Moreover, from \citet{Guo2018}, we know that $\lambda_2 \geq \lambda_2'$. So we can loosely bound the $R'_{ij}$ as follows:
\begin{equation}
    R_{ij} \leq R'_{ij} \leq \frac{2}{1-\lambda_2}
\end{equation}

In fact, this applies to any number of edge deletions, as long as $G'$ stays connected. Again, we can define a maximum possible change in resistance distance across the graph:
\begin{equation}
    \Delta_{del} = \max_{i,j\in V}(\frac{2}{1-\lambda_2} - R_{ij}) = \frac{2}{1-\lambda_2} - \min_{i,j\in V}(R_{ij})
\end{equation}
This leads to the following bounds on node and edge curvature, and completes the proof:
\begin{align*}
 p_i - \frac{d_i}{2}\Delta_{del} &\leq p'_i \leq p_i   \\
 (1-\lambda_2)\big[p_i+p_j - \frac{\Delta_{del}}{2}(d_i+d_j)\big] &\leq \REC'(i,j) \leq \REC(i,j) \\
    \REC(i,j) - \frac{1}{R_{ij}+\Delta_{del}} \big[\frac{2}{R_{ij}}(2R_{ij}+\Delta_{del})(p_i+p_j) - \Delta_{del}(d_i+d_j)\big] &\leq \REC'(i,j) \leq \REC(i,j) 
\end{align*}
\end{proof}

\subsection{Expressivity Proofs}

\thmstabilitylowerbound*

\begin{proof}
  Considering the calculation of persistence diagrams based on
  scalar-valued filtrations functions, every point in the persistence
  diagram~$D_f$ can be written as a tuple of the form~$(f(e_F), f(e_F'))$,
  with $e_F, e_F' \in E_F$; the sample applies for $D_g$. The inner distance
  between such tuples that occur in the bottleneck distance calculation
  can thus be written as
  \begin{equation}
    \|(f(e_F), f(e_F')) - (g(e_G), g(e_G'))\|_\infty,
  \end{equation}
  which we can rewrite to $\max_{C\colon E_F \to E_G}\{ f(x) - g(C(x))\}$
  for a general map~$C$ induced by the bijection of the
  bottleneck distance. Not every map is induced by a bijection, though.
  Hence, if we maximise over \emph{arbitrary} maps between the edge
  sets, we are guaranteed to never exceed the bottleneck distance.
\end{proof}

\section{Additional Proofs for Distinguishing Strongly-Regular Graphs}

\thmExpressivityCurvature*

\begin{proof}
  We first show the part of the statement relating to the
  \emph{Forman--Ricci curvature}.
  Given a distance-regular graph $G$ with $N$ vertices and
  intersection array $\{ b_0, b_1, \dots , b_{D-1}; c_1, c_2, \dots
  , c_D\}$. Let $i, j$ be adjacent nodes in $G$. For a regular graph,
  we have $d_i = d_j = b_0$, where $b_0$ is a constant. The number
  of triangles between two adjacent nodes $i$ and $j$ in $G$ is given
  by $a_1 = b_0 - b_1 - c_1$~\citep{Van_Dam_2016}. The Forman curvature
  of $i, j$ is thus
  \begin{equation}
    \FRC(i, j) := 4 - 2b_0 + 3|b_0 - b_1 - c_1|.
  \end{equation}
  Given two strongly-regular graphs with the same intersection array,
  i.e.\ the same values of $b_0$, $b_1$ and $c_1$, the Forman curvature
  yields the same value for all pairs of adjacent nodes and cannot
  distinguish them.
  For the \emph{resistance curvature}, the claim follows as an immediate
  Corollary of Theorem~A~\citep{biggs_1993} and described in
  \citet{KOOLEN2013770}. Given the resistance between two nodes depends
  only on the intersection array and the number of nodes in the graph,
  then the resistance curvature cannot distinguish two strongly-regular
  graphs.

  The expressivity of Ollivier--Ricci curvature is strictly better, and
  it turns out that there are graphs with the same intersection array
  that we can distinguish, namely the Rook graph and the Shrikhande
  graph.
  Both graphs have the same intersection array $\{6, 3; 1, 2\}$ but differ
  in their first hop peripheral subgraphs~\citep{khop-gnn}.
  It is known that $2$-WL cannot distinguish these graphs. Ollivier--Ricci
  curvature, however, is sensitive to these differences in peripheral
  subgraphs with the edge curvatures for the Rook graph being: 
  {
    \let\ab=\allowbreak
  $[0.2,\ab 0.2,\ab 0.33,\ab 0.33,\ab 0.33,\ab 0.2,\ab 0.33,\ab 0.33,\ab 0.33,\ab 0.33,\ab  0.33,\ab 0.33,\ab 0.33,\ab 0.33,\ab 0.33,\ab 0.2,\ab 0.2,\ab 0.33,\ab 0.33,\ab 0.33,\ab  0.33,\ab 0.33,\ab 0.33,\ab 0.33,\ab  0.33,\ab 0.33,\ab 0.33,\ab 0.33,\ab 0.33,\ab 0.2,\ab 0.33,\ab 0.33,\ab 0.33,\ab 0.33,\ab 0.33,\ab 0.33,\ab 0.33,\ab 0.33,\ab 0.33,\ab 0.33,\ab 0.33,\ab 0.33]$, and for the Shrikhande graph they are
  $[0,\ab 0,\ab 0.27,\ab 0.27,\ab 0.1,\ab 0,\ab 0.27,\ab 0.27,\ab 0.1,\ab 0,\ab 0.27,\ab 0.1,\ab 0.27,\ab 0,\ab
  0.27,\ab 0.1,\ab 0.27,\ab 0,\ab 0.1,\ab 0.27,\ab  0.27,\ab 0.1,\ab 0.27,\ab 0.27,\ab 0.17,\ab  0.17,\ab 0.17,\ab 0.17,\ab 0.17,\ab 0.17,\ab 0.17,\ab 0.17,\ab 0.17,\ab 0.17,\ab 0.17,\ab 0.17,\ab 0.17,\ab 0.17,\ab 0.17,\ab 0.17,\ab 0.17,\ab 0.17]$,\ab
  }
  demonstrating that OR curvature can distinguish these
  graphs---\emph{unlike Resistance curvature, Forman--Ricci curvature
  and the $2$-WL test}.
\end{proof}

\section{Additional Stability Analysis}\label{app:Stability Analysis}

Given the bounds on curvature established in \cref{sec:Stability}, we explore how curvature changes experimentally by analysing edge perturbations on Erdős–Rényi graphs. In particular, we provide statistics that quantify the maximal change in curvature for random graphs with varying connectivity parameters in response to edge additions and deletions. The experiment fixes the number of nodes in the ER graphs ($n=100$), and generates a sample of 50 graphs for the selected values of $p$. For each graph in the sample, we measure the curvature $\kappa $  of all edges and calculate the standard deviation $\sigma_\kappa$ of this distribution. We then perturb the original graph by edge addition/deletion and calculate the new curvature $\kappa'$. The following tables present the \emph{worst case} deviations in curvature, which we define as $\Delta\kappa = |\kappa - \kappa'|$, in units of $\sigma_\kappa$; in other words the maximal value of $\Delta\kappa/\sigma_\kappa$ over all sample graphs and their edges.    

 \begin{table}[H]
\small
\begin{tabular}{l|lllllllll}
\multirow{3}{*}{Curvature}            & \multicolumn{9}{l}{{\textbf{Edge Addition:} Maximal Change ($\downarrow$) in Curvature for ER Graphs} 
($\Delta\kappa / \sigma_\kappa$)}        \\
                                     & \textit{p}=0.1 & \textit{p}=0.2 & \textit{p}=0.3 & \textit{p}=0.4 & \textit{p}=0.5 & \textit{p}=0.6 & \textit{p}=0.7 & \textit{p}=0.8 & \textit{p}=0.9 \\ \hline
\toprule
\FRC                                 
&0.582
&0.419
&0.334
&0.296
&0.253
&0.246
&0.232
&0.25
&0.315  \\
\hline
\ORC         
&1.545
&1.157
&0.613
&0.465
&0.396
&0.366
&0.368
&0.399
&0.512 \\
\hline
\REC        
&0.689
&0.417
&0.296
&0.251
&0.221
&0.232
&0.227
&0.243
&0.321\\
\hline        
\end{tabular}
\end{table}

\begin{table}[H]
\small
\begin{tabular}{l|lllllllll}
\multirow{3}{*}{Curvature}            & \multicolumn{9}{l}{{\textbf{Edge Deletion:} Maximal Change ($\downarrow$) in Curvature for ER Graphs} ($\Delta\kappa / \sigma_\kappa$)}        \\
                                     & \textit{p}=0.1 & \textit{p}=0.2 & \textit{p}=0.3 & \textit{p}=0.4 & \textit{p}=0.5 & \textit{p}=0.6 & \textit{p}=0.7 & \textit{p}=0.8 & \textit{p}=0.9 \\ \toprule
\FRC                                 
&0.609
&0.408
&0.334
&0.291
&0.255
&0.239
&0.245
&0.243
&0.319  \\
\hline
\ORC         
&1.397
&1.27
&0.623
&0.479
&0.394
&0.365
&0.347
&0.402
&0.482 \\
\hline
\REC        
&0.75
&0.431
&0.336
&0.248
&0.229
&0.218
&0.225
&0.242
&0.312
\\
\hline        
\end{tabular}
\end{table}



\section{Additional Commentary on Counting Substructures} \label{app:sec:CountingSubstructs}
 We find that the difference in perspective between the selected curvature notions
is underscored by their respective performance when counting substructures.
\emph{Forman curvature} is an
inherently local measure by definition, depending only on $3$-cycles
between adjacent nodes and their degrees. \emph{Ollivier--Ricci curvature}, when
used with a uniform measure, can bound the number of triangles
within a locally finite graph~\citep{OR_clustering} through its relation
with the Watts--Strogatz clustering coefficient~\citep{Watts1998}. It
can also be shown that quadrangles and pentagons influence the OR
curvature, further enhancing the expressivity of this type of
curvature~\citep{OR_clustering}.

This is the most global perspective one can achieve
using $\ORC$ with uniform probability measures,
since polygons with more than five edges do not impact the
curvature valuation.

However, by changing the probability measure used by
$\ORC$, we can shift the focus towards even larger substructures. For
example, the $n$th power of the transition matrix provides information
on the number of $n$-paths and can therefore provide substructure
information for cycles of size $n$~\citep{distance-encoding}. 
\emph{Resistance curvature}, by contrast, is biased towards the largest
substructures. Due to the `global' nature of the
resistance distance metric, $\REC$ assigns cycles of size $\geq$
5 a positive curvature. Moreover, in a locally finite graph, one cannot
use $\REC$ to establish a non-trivial bound on the number of triangles
(consider creating an infinite cycle between two nodes).

\section{Probability Measure for Ollivier--Ricci Curvature and Counting
Substructures}\label{app:OR}

The Ollivier--Ricci curvature is of particular interest because of its
flexibility. While the predominant probability measure~$\mu$ used by the
community is \emph{uniform} for each node, i.e.\ each of the node's
neighbours is chosen with probability being proportional to the degree
of the node. We experimented with different probability measures, one
being based on expanding~$\mu$ to the two-hop neighbourhood of a vertex,
the other one being based on random walk probabilities. Specifically,
for a node~$x$ and a positive integer~$m$, we calculate
$\mu_{\text{RW}}$ as
\begin{equation}
  \mu_{\text{RW}}(y) := \sum_{k \leq m} \phi_{k}(x, y),
  \label{eq:Random walk probability}
\end{equation}
with $\phi_{k}(x, y)$ denoting the probability of reaching node~$y$ in
a $k$-step random walk that starts from node~$x$. Subsequently,  we
normalise \cref{eq:Random walk probability} to ensure that it is a valid
probability distribution. In our experiments, we set $m = 2$, meaning
that at most $2$-step random walks will be considered. As shown in the
main paper, this formulation leads to an increase in expressivity, and
we expect that further exploration of the probability measures will be
a fruitful direction for the future.

We now explore to what extent the ability of the curvature
to count substructures can also be improved in this way. To do this, we
used powers of the transition matrix as the probability measure, as it
has been shown that the $n$th power provides information on the number
of $n$-paths and can therefore provide substructure information for
cycles of size $n$~\citep{distance-encoding}. We find that powers of
the transition matrix larger than 1 can perform better for counting the
substructures, particularly for substructures larger than 3-cycles.
There is also a difference between Regular and Erdős–Rényi (ER) graphs
as the best transition power tends to be higher for ER graphs. We
hypthosise that this may have something to do with the mixing time of
the graph, as large powers should converge to the stationary
distribution, and regular graphs are more `expander-like'. The best
results are obtained by taking multiple landscapes using the transition
matrix powers (up to $n = 5$) and then averaging them. We show that
combined with a single layer MLP, this method can perform better than
using Graph Neural Network based approaches and OR curvature with the
uniform measure.  

 \begin{table}[h]
\small
\begin{tabular}{l|llll}
\multirow{2}{*}{Method}              & \multicolumn{4}{l}{Counting Substructures (MAE $\downarrow$)}         \\
                                     & Triangle        & Tailed Tri.     & Star            & 4-Cycle         \\ \hline
GCN                                  & 0.4186          & 0.3248          & \textbf{0.1798}          & 0.2822         \\
\hline
\ORC \ Filtration             & 0.2321         & 0.2395          & 0.3393          & 0.3089           \\ 
\ORC \ Filtration with transition matrix powers             & \textbf{0.1956}          & \textbf{0.2095}          & 0.3212          & \textbf{0.2680}           \\
\hline        
\end{tabular}
\end{table}

\begin{table}[h]
\begin{tabular}{l|l|l}
\multirow{2}{*}{Method}              & \multicolumn{2}{l}{Optimal Transition Power}         \\
                                     & ER        & Regular         \\ \hline
Triangle                                  & 2          & 1              \\
Tailed Triangle                                  & 4          & 3     \\
Star                     & 4         &  2         \\
Chordal Cycle                 & 2          & 2   \\
4-Cycle             & 8          &  3  \\ \hline
\end{tabular}
\end{table}

\section{Computational Complexity}\label{app:sec:Computational Complexity}

\begin{table}[tbp]
  \centering
  \sisetup{
    table-format    = 5.4,
    round-mode      = places,
    round-precision = 3,
  }
  \caption{Computation time in seconds for discrete curvature on varying Erdős–Rényi graph sizes with $p = 0.3$.}
  \label{tab:Computational complexity}
  \begin{tabular}{S[table-format=4, round-precision = 0]SSS}
    \toprule
    {Number of Nodes} & {\FRC} & {\ORC} & {\REC}\\
    \midrule
    10 & 0.000041 & 0.0015 & 0.02 \\
    50 & 0.00076 & 0.038 & 0.70 \\
    100 & 0.0047 & 0.247 & 6.61 \\
    250 & 0.054 & 4.72 & 252.85 \\ 
    500 & 0.38 & 59.27 & 6414.97 \\
    1000 & 2.92 & 1040.70 & 74366.07 \\
    \bottomrule
  \end{tabular}
\end{table}

\begin{table}[tbp]
  \centering
  \sisetup{
    table-format    = 5.4,
    round-mode      = places,
    round-precision = 1,
  }
  \small
  \caption{Computation time for different numbers of Erdős-–Rényi graphs in a reference set ($n = 10$ and $p = 0.3$) with different distribution distance measures}
  \label{tab:Computational complexity 2}
   \begin{tabular}{S[table-format=4, round-precision = 0]SSSS}
    \toprule
    {Number of Graphs} & {Degree + MMD} & {Orbit + MMD} & {Curvature + MMD} & {Curvature + Landscapes} \\
    \midrule
20                         & 3.64ms       & 217ms       & 4.24ms          & 12.5ms                         \\ 
50                         & 10.9ms       & 459ms       & 12.4ms          & 20.9ms                         \\ 
100                        & 34.1ms       & 887ms       & 37.9ms          & 34.4ms                         \\ 
200                        & 120ms        & 1960ms      & 133ms           & 80.4ms                         \\ 
500                        & 678ms        & 6740ms      & 727ms           & 144ms                          \\ 
1000                       & 2620ms       & 19900ms     & 2680ms          & 359ms                          \\

    \bottomrule
  \end{tabular}
\end{table}

\begin{table}[tbp]
  \centering
  \sisetup{
    table-format    = 5.4,
    round-mode      = places,
    round-precision = 1,
  }
  \caption{Computation time for a fixed number of Erdős–Rényi graphs in a reference set with different sizes ($p = 0.3$) with different distribution distance measures}
  \label{tab:Computational complexity 3}
   \begin{tabular}{S[table-format=4, round-precision = 0]SS}
    \toprule
    {Number of Graphs} & {Curvature + MMD} & {Curvature + Landscapes} \\
    \midrule
10             & 2.27ms          & 9.04ms                         \\ 
20             & 4.45ms          & 12.5ms                         \\ 
50             & 15.8ms          & 20.9ms                         \\ 
100            & 93.6ms          & 34.4ms                         \\ 
200            & 556ms           & 80.4 ms                        \\ 
500            & 727ms           & 144ms                          \\ 
1000           & 2800ms          & 364ms                          \\ 
    \bottomrule
  \end{tabular}
\end{table}


Persistence diagrams of $1$-dimensional simplicial complexes, i.e.\
graphs, can be computed in $\mathcal{O}(m \log m)$ time where~$m$ denotes the
number of edges.
Empirically, when calculating different curvature measures for different
sizes of graphs, we find that Forman curvature scales well to large
graphs, whereas OR and resistance curvatures can be used for smaller
graphs and in cases that require a more expressive measure. Note that
there are significantly faster ways to calculate resistance curvature as
an approximation \cite{affinity-gnn}. A majority of works on GGMs focus
on small molecule generation, where any of these curvatures can be used
with minimal pre-computation. 
\Cref{tab:Computational complexity} depicts the computational complexity
of various curvature calculations on Erdős–Rényi graphs whilst \Cref{tab:Computational complexity 2} and \Cref{tab:Computational complexity 3} compares the complexity to methods based on MMD. We find that calculating persistence diagrams, turning these to persistence landscapes, averaging these and then calculating a distance takes a similar amount of time compared to MMD for different sizes of graphs and for different numbers of graphs in the reference set. Interestingly, our approach scales better than MMD as both the number of graphs in the reference set increases and when the size of the graphs increases. This will be important for comparing distributions of large data sets such as the commonly used Zinc dataset or QM9. Overall, we find that our method can be easily applied in practical use cases, especially given that models for graph generation tyically generate graphs with well under 1000 nodes.

\section{Ethical Concerns}\label{app:sec:Ethical Concerns}
We have proposed a general framework for comparing graph distributions focusing primarily on method and theoretical development rather than on potential applications. We currently view drug discovery as being one of the main application areas, where further experiments may be required, but we have no evidence that our method enhances biases or causes harm in any way. 

\clearpage
\appendix
\onecolumn

\end{document}